\documentclass[runningheads,a4paper]{llncs}
\usepackage{graphicx}
\usepackage{amsmath}
\usepackage{amssymb}

\usepackage[ruled,vlined]{algorithm2e}

\usepackage{subcaption}
\usepackage{url}
\urldef{\mailsa}\path|dhifli.wajdi@courrier.uqam.ca|
\urldef{\mailsb}\path|diallo.abdoulaye@uqam.ca|
\urldef{\site}\path|https://sites.google.com/site/wajdidhifli/softwares/galaxy-x|  
\newcommand{\keywords}[1]{\par\addvspace\baselineskip
\noindent\keywordname\enspace\ignorespaces#1}

\begin{document}

\mainmatter  

\title{Toward an Efficient Multi-class Classification in an Open Universe}

\titlerunning{Toward an Efficient Multi-class Classification in an Open Universe}

%
%
\author{Wajdi Dhifli\and Abdoulaye Banir\'e Diallo
\thanks{This paper has received the best paper award at the 12th International Conference on Machine Learning and Data Mining MLDM 2016, New York, USA.
}
}
\authorrunning{Dhifli et \textit{al.}}

\institute{Department of Computer Science, University of Quebec At Montreal,\\
PO box 8888, Downtown station, Montreal (Quebec) Canada, H3C 3P8.\\
\mailsa\\
\mailsb\\
}

%
%

\toctitle{Toward an Efficient Multi-class Classification in an Open Universe}
\tocauthor{Wajdi Dhifli, Abdoulaye Banir\'e Diallo}
\maketitle

\begin{abstract}
Classification is a fundamental task in machine learning and data mining. Existing classification methods are designed to classify unknown instances within a set of previously known training classes. Such a classification takes the form of a prediction within a closed-set of classes. However, a more realistic scenario that fits real-world applications is to consider the possibility of encountering instances that do not belong to any of the training classes, $i.e.$, an open-set classification. In such situation, existing closed-set classifiers will assign a training label to these instances resulting in a misclassification. In this paper, we introduce \textit{Galaxy-X}, a novel multi-class classification approach for open-set recognition problems. For each class of the training set, Galaxy-X creates a minimum bounding hyper-sphere that encompasses the distribution of the class by enclosing all of its instances. In such manner, our method is able to distinguish instances resembling previously seen classes from those that are of unknown ones. To adequately evaluate open-set classification, we introduce a novel evaluation procedure. Experimental results on benchmark datasets show the efficiency of our approach in classifying novel instances from known as well as unknown classes.

\end{abstract}

\section{Introduction}
Classification is a central task in machine learning and data mining. It has an extremely large number of domains of application ranging from finance and marketing, to bioinformatics and computer vision \cite{alex_2012}\cite{maxwell_2015}. 
The most conventional classification scenario is to train a classifier on a set of instances of known classes, $i.e.$ the training set, then to predict the class label of unknown instances within the same set of already seen classes \cite{quinlan_1987}\cite{cortes_1995}\cite{tax_2001}\cite{rocha_2014}. Such a classification takes the form of a prediction within a \textit{closed-set} of labels. However, due to the growth of data collection in many real-world applications, the training data could only represent a partial view of the domain and thus it may not contain training examples for all possible classes. 
%
In such scenario, the classifier may be confronted to observations that do not belong to any of the training classes. In this context, the classification becomes within an \textit{open-set} of labels where the presence of observations from unseen classes is possible. In \textit{open-set classification}, traditional closed-set classifiers will fail in the prediction for observations of unseen labels. 

In applications where the user is interested in the identification of few classes from a large classification universe, the most conventional way is to fuse the set of uninteresting classes into one single large negative set which usually makes the dataset highly unbalanced. In this case, the classifier becomes overwhelmed by negative observations which hinders the discrimination of positive classes. Some attempts have emerged trying to remedy such situation, mainly based on the sampling of a subset of representatives from the negatives \cite{raskutti_2004}. However, it is very difficult and somehow unfair to reduce all the negatives into a small summary that may not be sufficient to represent all possibilities. A more appropriate transformation of such problem is an open-set classification where only positive classes are modeled in training and any observation that remarkably deviates from the distribution of known classes is rejected.

Several domains of application of open-set classification exist. In bioinformatics, the advances in sequencing technology have made the acquisition of genomic sequences fast and easy \cite{maxwell_2015}. While a virologist analyses the genomic sequence of a virus, he always keeps the possibility that the latter can be an unknown one. A closed-set classification does not help in such situation since the new virus will be assigned a previously known type (label) preventing a discovery. Another example of domains of application is computer vision \cite{walter_2013}. For instance, in face identification, the system is only interested in recognizing a number of faces within an infinite set of possibilities, $i.e.$ an open universe of classes. In similar applications, the classifier should be able to create a decision boundary that envelops the class instances and resembles its distribution such that whatever lies outside of the class boundary is rejected.

In this paper, we introduce \textit{Galaxy-X}, an open-set multi-class classification approach. For each class, Galaxy-X creates a minimum bounding hyper-sphere that encloses all of its instances. In such manner, it is able to distinguish between novel instances that fit the distribution of a known class from those that diverge from it. Galaxy-X introduces a softening parameter for the adjustment of the minimum bounding hyper-spheres to add more generalization or specialization to the classification models. To properly evaluate open-set classification, we also propose a novel evaluation technique, namely \textit{Leave-P-Class-Out-Cross-Validation}. Experimental evaluations on benchmark datasets show the efficiency of Galaxy-X in open-set multi-class classification.


\section{Related Work}\label{sec:related_work}

Very few works have addressed open-set classification in the literature. Scheirer et al. presented a formalization of open-set classification and showed its importance in real-world applications \cite{walter_2013}. The authors discussed the bias related to the evaluation of learning approaches and how recognition accuracies are inflated in closed-set scenarios, leading to an over-estimated confidence in the evaluated approaches \cite{walter_2013}\cite{torralba_2011}. 
In binary closed-set classification, SVM defines a hyper-plane that best separates between the two classes. 
Scheirer et al. proposed an SVM based open-set multi-class classifier termed one-vs.-set SVM \cite{walter_2013}, which 
defines an additional hyper-plane for each class such that the latter becomes delimited by two hyper-planes in feature space. A testing instance is then classified as of one training class or as of an unknown class, depending on its projection in feature space. Although this strategy delimits each training class from two sides, the class "acceptance space" is left unlimited within the region between the hyper-planes and no additional separator is provided to prevent misclassifying unknown instances that are within the class hyper-planes bound but far away from the distribution of its training instances in feature space.
%

Another important learning approach for open-set problems is one-class classification. The most known technique is one-class SVM \cite{tax_2001} where the classifier is trained only on a single positive class and the aim is to define a contour that encloses it from the rest of the classification universe. Any instance that lies outside of the defined class boundary is considered as negative. One-class classification is mainly used in outlier and novelty detection. It is limited to single class classification and cannot be directly used in multi-class classification.

One-vs.-one and one-vs.-rest \cite{rocha_2014} are popular techniques for multi-class classification. One-vs.-one constructs a model for each pair of classes. Then, test examples are evaluated against all the constructed models. A voting scheme is applied and the predicted label is the one with the highest number of votes. One-vs.-rest creates a single classifier per class, with the examples of that class as positives and all the other examples as negatives. All classifiers are applied on a test example and the predicted label is the one with the highest score. It is possible to consider one-vs.-rest for open-set classification by iteratively using each class as the positive training set, and all the remaining (known) classes as the rest of the universe. However, in open-set classification, the classification universe is unlimited and thus the classifier will suffer a negative set bias. 

Based on \cite{Landgrebe_2005} and \cite{Tax_2008}, it is possible to build a simple open-set multi-class classifier using a combination of a one-class classifier and a multi-class classifier. In the first step, all training classes are fused into a single large "super-class" and the one-class classifier is trained on the entire super-class. In this setting, the one-class classifier will directly reject and label as unknown all testing instances that do not fit the distribution of all known training classes. In the second step, the multi-class classifier is trained on the original training classes and is used to classify instances that were not rejected by the one-class classifier. 

\section{Galaxy-X}\label{sec:galaxy-x}
\subsection{Preliminaries and Problem Definition}
Let $\mathcal{D}$ be a training set of $n$ instances and $\mathcal{L}$ be the set of possible labels in $\mathcal{D}$, $\mathcal{D}=\{(x_1, l_1), ..., (x_n, l_n)\}$ where $l_i \in \mathcal{L}$ and $x_i$ is defined by a vector in $d$-dimensional space, $\forall i \in [1,n]$. 
In open-set classification, the classifier should be able to assign to a test instance $x$ a label $l_x$ that is known $l_x \in \mathcal{L}$ or that is \textit{unknown}, $i.e.$, $l_x \in \mathcal{L}\cup \{"unknown"\}$. In this setting, it is necessary to define a boundary envelop for each class in order to make it distinguishable from other unknown possibilities. The definition of such boundary is difficult as the delimited-class-space should reflect the class distribution by enclosing as many as possible of its instances while keeping outside as many as possible of the rest of instances. Indeed, this can be seen as an optimization problem of the classification error that considers a trade-off between generalization and specialization. As a possible solution, we define the minimum bounding hyper-sphere $\mathcal{M}$ as the smallest hyper-sphere that circumscribes all instances of a considered class. For a class $\mathcal{D}_l \subseteq \mathcal{D}$ of label $l\in \mathcal{L}$, the hyper-sphere $\mathcal{M}_l$ represents the class model that resembles the distribution of $\mathcal{D}_l$ instances. Each class model $\mathcal{M}_l$ ($\forall l\in \mathcal{L}$) is defined as: 
\begin{equation}\label{eq:model}
\mathcal{M}_l = (c_l, r_l), \forall l \in \mathcal{L}, r_l>0
\end{equation} 
where $c_l$ is the center of $\mathcal{M}_l$ hyper-sphere (the class mean $\overline{x}$):
\begin{equation}\label{eq:center}
c_l = \overline{x}, \forall x_i \in \mathcal{D}_l, \forall l \in \mathcal{L}
\end{equation}
and $r_l$ is the radius of $\mathcal{M}_l$ hyper-sphere, $i.e.$, the distance between $c_l$ and the most divergent instance from $\mathcal{D}_l$ (the class variance):
\begin{equation}\label{eq:radius}
r_l = \max{(\Delta(x_i, c_l))}, \forall x_i \in \mathcal{D}_l, \forall l \in \mathcal{L}
\end{equation}
where $\Delta$ is a function returning the distance between $c_l$ and $x_i$ with respect to a distance measure. In a multi-class scenario, the resulting representation space is similar to a galaxy of classes in an open universe of possibilities.
%
\subsection{The Training Process}
Algorithm \ref{alg:Galaxy-X_training} describes the training phase in Galaxy-X. Given a training set $\mathcal{D}$ and a training label set $\mathcal{L}$ over $\mathcal{D}$, we create a model $\mathcal{M}_l$ for each class $l$ $\in \mathcal{L}$ that is composed of the class minimum bounding hyper-sphere center $c_l$ and radius $r_l$ as defined in Equations \ref{eq:model}, \ref{eq:center} and \ref{eq:radius}.
\begin{algorithm}[!t]
\caption{Galaxy-X: The training process}\label{alg:Galaxy-X_training}
\LinesNumbered
\KwData{$\mathcal{D}$: training set, $\mathcal{L}$: training labels}
\KwResult{$\mathcal{M}$: set of class models}
\SetKwFunction{Boundary}{Boundary}\SetKwFunction{Centeroid}{Centeroid}

\Begin{
$\mathcal{M} \leftarrow \emptyset$\\
\ForEach{($l\in \mathcal{L}$)}{
$c_l \leftarrow $ \Centeroid($D_l$)\\
$r_l \leftarrow $ \Boundary($D_l$)\\
$\mathcal{M}_l  \leftarrow(c_l, r_l)$\\
$\mathcal{M} \leftarrow \mathcal{M} \cup \mathcal{M}_l$\\
	}
}
\end{algorithm}
\subsection{Acceptance of Instances}
In open-set classification, the classifier should discriminate between instances of known classes and reject those of unknown ones. We define a score of acceptance of an instance by a class depending on its position from the class boundary. 
\begin{definition}(Acceptance Score)\label{def:acc_score}{ The acceptance score, denoted by $\phi$, for an instance $x$ by a class of label $l \in \mathcal{L}$, is defined as follows:
\begin{equation}
\phi(x, l) =  1-\frac{\Delta(x, c_l)}{r_l}
\end{equation}
where $\Delta$ is a distance measure, $c_l$ is the class center, and $r_l$ is its radius.
}\end{definition}
The acceptance score is defined in $]-\infty, 1]$ ($i.e. \phi \in{\rm I\!R}_{\leq 1}$). It allows to decide whether an instance is accepted or rejected by a class. It is interpreted as follows:
\begin{itemize}
\item $\phi\in [0, 1]$: the query instance $x$ is accepted by the class $l$:
	\begin{itemize}
		\item $\phi \in ]0,1[$: $x$ is inside the hyper-sphere of $l$,
		\item $\phi = 1$: $x$ is in the class center, $i.e.$, $x = c_l$,
		\item $\phi = 0$: $x$ is on the boundary, $i.e.$, $distance(x, c_l)=r_l$.
	\end{itemize}
\item $\phi < 0$: $x$ is out of the class boundary (rejected). The lower is the score, the farther is $x$ from the distribution of $l$.
\end{itemize}

Galaxy-X minimizes the classification error (\textit{Err}) that can be formulated as:
\begin{equation}\label{eq:classif_error}
Err = \sum_{\forall l\in\mathcal{L}}\sum_{\forall x_i\in\mathcal{D}} \psi(x, l)
\end{equation}
where $\psi\in[0,1]$ is a binary function that is defined as follows: 
\begin{equation} \label{eq:F_acceptance_score}
  \psi(x, l)=
  \begin{cases}
    1, & \text{if } \phi(x, l)\in [0,1] \text{ and } x\in\mathcal{D}_l , \text{ or} \\
    & \text{if } \phi(x, l) < 0 \text{ and } x\notin\mathcal{D}_l \\
    0, & \text{otherwise.}
  \end{cases}
\end{equation}

\subsection{The Classification Process}
\subsubsection{Filtering Prediction Candidate labels}\label{sec:filtering_labels}
Based on the acceptance score, it is possible, for a query instance $x$, to filter a subset of candidate labels $\mathcal{L}_x \subseteq \mathcal{L}$. The latter is the subset of remaining possible candidates, such that if $\mathcal{L}_x \neq \emptyset$, then the predicted label $l_x \in \mathcal{L}_x$. Algorithm \ref{alg:Galaxy-X_filtering} describes the general algorithm of filtering of the candidate labels. It starts with an empty set of candidate labels. Given the training class models, it tests whether the query instance $x$ is accepted or rejected by each training class according to Definition \ref{def:acc_score}. Indeed, it rejects all class labels where $x$ lies outside of the class boundary and only the subset of labels where $x$ is accepted is retained as the set of candidate labels for prediction.
\begin{algorithm}[!t]
\caption{Galaxy-X: The label filtering process}
\label{alg:Galaxy-X_filtering}
\LinesNumbered
\KwData{$\mathcal{M}$: set of class models, $x$: test instance}
\KwResult{$\mathcal{L}_x$: retained candidate labels}
\SetKwFunction{AcceptanceScore}{AcceptanceScore}
\Begin{
$\mathcal{L}_x \leftarrow \emptyset$\\
\ForEach{(class model $\mathcal{M}_l\in \mathcal{M}$)}{
\If{$\phi(x, l) \geq 0$}{
$\mathcal{L}_x \leftarrow \mathcal{L}_x$ $\cup$ $l$\\
}
}
}
\end{algorithm}

\subsubsection{Handling Class Overlapping}\label{sec:class_overlapping}
It is possible to obtain a set of disjoint hyper-spheres in the case where training classes are perfectly separable. In such case, if a query instance $x$ is circumscribed a hyper-sphere then $x$ takes the latter's class label otherwise $x$ is considered as of an unknown class. 
However, in real-world cases the hyper-spheres may overlap mainly in the presence of high inter-class similarity. 
In fact, the overlapping space between classes resembles a local closed-set classification within an open-set classification context. In this case, we train a local closed-set classifier only on the overlapping classes then we use it for only classifying query instances that are within the overlapping space, $i.e.$, instances that are accepted by multiple classes in Algorithm \ref{alg:Galaxy-X_filtering}, $|\mathcal{L}_x| > 1$. 

\subsubsection{The Classification Process}
Algorithm \ref{alg:Galaxy-X_classification} describes the classification process of Galaxy-X. The first step in prediction is the filtering of candidate labels according to Algorithm \ref{alg:Galaxy-X_filtering}. If the retained set of candidate labels is an empty set $\mathcal{F}^{L}_{x} = \emptyset$, then the query instance $x$ do not fit any training class. Thus, the predicted label $l_x$ is set to \textit{"Unknown"}. If $|\mathcal{F}^{L}_{x}| = 1$, $x$ is only accepted by one training class. In this case, the predicted label is that single filtered possibility $l_x \leftarrow \mathcal{F}^{L}_{x}$. In the case where $|\mathcal{F}^{L}_{x}| > 1$, $x$ shares similarities with more than one class and its feature vector is projected in the overlapping area between the hyper-spheres of the retained class labels. 
As this situation presents a conventional closed-set classification, a closed-set classifier $\mathcal{E}$ is locally trained only on the retained classes of $\mathcal{F}^{L}_{x}$, then $\mathcal{E}$ is used to predict the class label $l_x$ of $x$ such that $l_x \leftarrow \mathcal{E}(x)$ and $l_x \in \mathcal{F}^{L}_{x}$.%
\begin{algorithm}[!t]
\caption{Galaxy-X: The classification process}
\label{alg:Galaxy-X_classification}
\LinesNumbered
\KwData{$\mathcal{M}$: set of class models, $\mathcal{E}$: local closed-set classifier, $x$: test instance}
\KwResult{$l_x$: predicted label for $x$}
\SetKwFunction{FilterLabels}{FilterLabels}
\Begin{
$\mathcal{F}^{L}_{x} \leftarrow \FilterLabels(\mathcal{M}, x)$ \\
\eIf{$\mathcal{F}^{L}_{x} = \emptyset$}{
$l_x \leftarrow $ "Unknown"\\
}{
\eIf{$|\mathcal{F}^{L}_{x}| = 1$}{
$l_x \leftarrow \mathcal{F}^{L}_{x}$\\
}{
Train($\mathcal{E}$, $\mathcal{F}^{L}_{x}$)\\
$l_x \leftarrow \mathcal{E}(x)$\\
}
}
}
\end{algorithm}
\subsection{Softening Class Boundaries}\label{sec:softening}
In order to add flexibility to the models, we introduce a softening parameter $\delta\in {\rm I\!R}$ that allows to perform a distortion of the class boundary. Indeed, it allows to add more generalization or specialization to the classification models. Figure \ref{fig:softening_shrinking_class_boundary} shows respectively examples of positive and negative softening of a class boundary. In Figure \ref{fig:softening_shrinking_class_boundary}(a) a positive softening extends the minimum bounding hyper-sphere allowing to add more generalization to the model. Extending the class boundary may help in detecting test instances that are from the same class but slightly deviate from the training instances. In contrast, in Figure \ref{fig:softening_shrinking_class_boundary}(b) a negative softening is performed on the hyper-sphere by shrinking its radius which adds more specialization to the class model. Shrinking the class boundary may help in rejecting instances that do not belong to the class but are within the hyper-sphere near to the class boundary. In addition, it can be used to alleviate or remove overlapping between classes. If the softening is performed, its value has to be carefully chosen as an over-generalization engenders many false positives. In contrast, an over-specialization makes the model under-fit the class.
\begin{figure}[!t]
\centering
\includegraphics[width=0.65\textwidth]{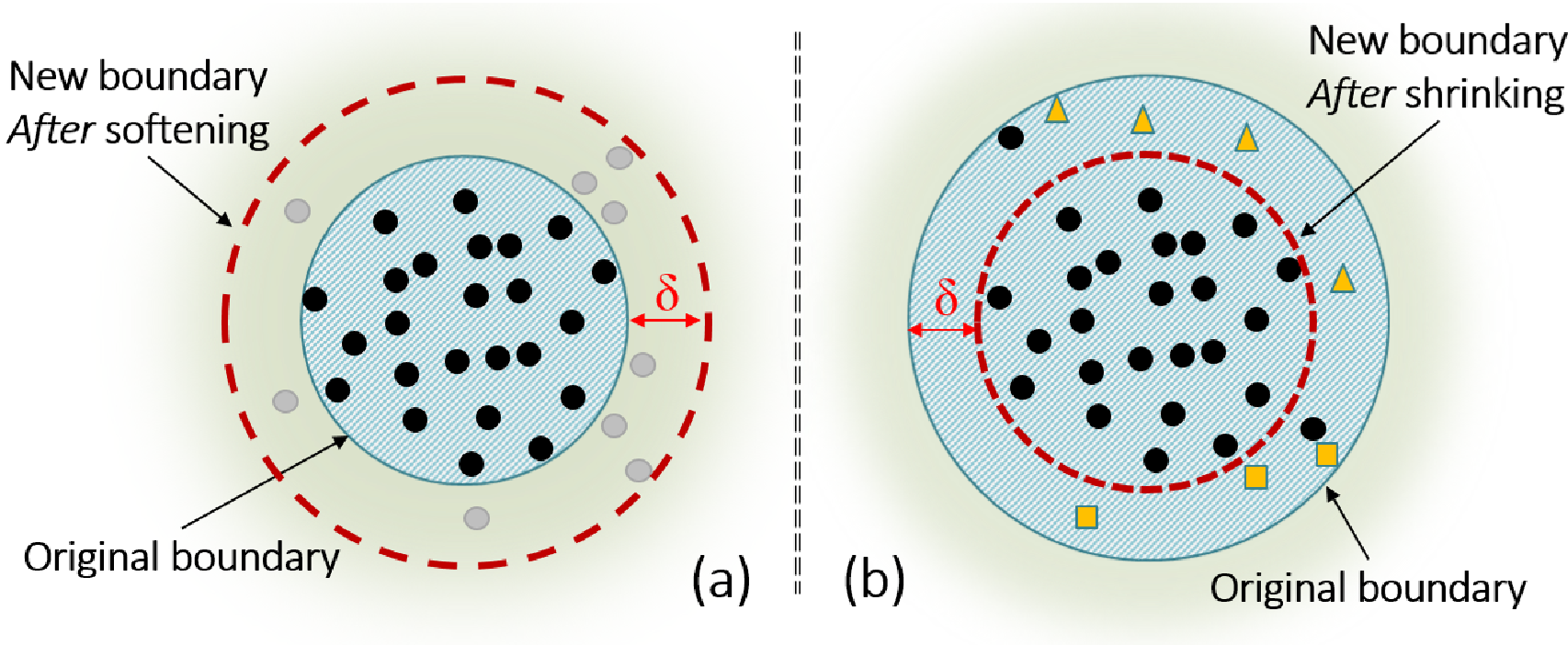}
\caption{Example of positive (a) and negative (b) softening of the class boundary. Black dots are the training class instances. Gray dots are test instances from the class which slightly differ from its distribution. Squares and triangles are negative examples and $\delta$ is the softening magnitude. Softening the boundary can help detecting true positives (a) or rejecting false positives (b).
}
\label{fig:softening_shrinking_class_boundary}
\end{figure}
%
\begin{definition}(Soft Acceptance Score)\label{def:soft_acc_score}{ The softening can be introduced in the acceptance score. We define the soft acceptance score ($\varphi$) as follows:
\begin{equation}\label{eq:soft_acceptance_score}
\varphi(x, l, \delta) =  1-\frac{\Delta(x, c_l)}{r_l + \delta} 
\end{equation}
where $\delta$ is the softening parameter, $distance$ is an appropriate distance measure, $c_l$ is the center of the class of label $l$, and $r_l$ is its radius.
}\end{definition}
Similarly to $\phi$, $\varphi$ is defined in $]-\infty, 1]$ and is interpreted in the same way. 
%
It is worth noting that softening can also be introduced in training (instead of $\varphi$) in the definition of class boundaries such that line 5 in Algorithm \ref{alg:Galaxy-X_training} becomes $r_l \leftarrow $ \Boundary($D_l$) + $\delta$. According to Equations \ref{eq:classif_error} and \ref{eq:soft_acceptance_score}, the optimal $\delta$ value, denoted $\delta^*$, should be the one that minimizes the classification error as follows:
\newcommand{\argminF}{\mathop{\mathrm{argmin}}\limits} 
\begin{equation}
\delta^* = \argminF_\delta Err = \argminF_\delta \sum_{\forall l\in\mathcal{L}}\sum_{\forall x_i\in\mathcal{D}} \Psi(x, l, \delta)
\end{equation}
where $\Psi$ is defined similarly to Equation \ref{eq:F_acceptance_score} but based on the $\varphi$. 

\begin{lemma} Given a classification scenario $S_\mathcal{D}$, a classification performance evaluation technique Perf, and a closed-set classifier $\mathcal{X}$: \begin{equation*}
\forall S_\mathcal{D}, Perf(\textit{Galaxy-}\mathcal{X}, S_\mathcal{D})\geq Perf(\mathcal{X}, S_\mathcal{D})
\end{equation*}

\end{lemma}
\begin{proof}
In the worst case, the optimal softening value will be very high until the training models completely overlap resembling a closed-set classification. In this case, evaluation instances will be classified using the local closed-set classifier. Consequently, \textit{Perf}(Galaxy-$\mathcal{X}$, $S_\mathcal{D})=$ \textit{Perf}($\mathcal{X}$, $S_\mathcal{D})$.
\end{proof}

\section{Experimental Evaluation}\label{sec:experimental_evaluation}
Evaluating open-set multi-class methods requires proper measures and protocols.

\subsection{How Open is an Open-set Classification?}
We propose \textit{Openness} to quantify the openness of a classification scenario ($S_\mathcal{D}$). 
\begin{definition}(Openness)\label{def:openness}{ It measures the ratio of labels that are unseen in training but encountered in prediction to all the labels of the dataset $\mathcal{D}$. Formally:
\begin{equation}
openness(S_\mathcal{D}) = \frac{|Unseen Labels|}{|\mathcal{L}|}
\end{equation}
}\end{definition}
\textit{Openness} is defined in ${\rm I\!R}^+$. An \textit{openness} of 0 means that it is a closed-set classification, otherwise it is an open-set classification. Theoretically, the value of \textit{openness} can be even $+\infty$ which means an infinite set of possibilities. 
However, in practical cases, the number of test labels can usually be delimited. In our experiments, \textit{openness} $\in [0,1[$, where the open-set classification will be simulated from a dataset were all possible labels are known, $i.e.,$ $\mid\mathcal{L}\mid$ = $\mid$\textit{TrainingLabels}$\mid$ + $\mid$\textit{UnseenLabels}$\mid$. An \textit{openness} of 1 means that $\mid$\textit{TrainingLabels}$\mid$ = 0. This corresponds to a clustering context which is out of the scope of this work.

\subsection{Evaluation Technique}
Conventional evaluation techniques are not suitable for open-set classification and they do not present sufficient restrictions on the labels to simulate an open-set classification evaluation. We propose \textit{Leave-P-Class-Out-CrossValidation} for open-set classification. It allows to simulate an open-set classification where we do not have knowledge of all prediction classes. Algorithm \ref{alg:LeavePClassOut-CrossValidation} describes the general procedure of \textit{Leave-P-Class-Out-CrossValidation}. First, all possible combinations $\mathcal{C}$ of $\mathcal{P}$ labels from $\mathcal{L}$ are computed. In each iteration, one combination \textit{comb} is randomly chosen from $\mathcal{C}$ without replacement. All instances of a label $l_{comb}$, $\forall l_{comb}\in $\textit{comb}, are discarded from the dataset $\mathcal{D}$ to be directly added to the test set. These instances are referred to as the \textit{Leave-out-instances}. All labels in \textit{comb} are unseen in training but encountered in testing which simulates an open-set classification. A $\mathcal{N}$-fold-cross-validation is performed on the remaining instances, $\mathcal{D}\setminus$\textit{Leave-out-instances}, where in each cross validation the \textit{Leave-out-instances} are directly added to the test set. The evaluation is repeated until a maximum number of iterations $\alpha$ is reached or no more combination is possible. 
\begin{algorithm}[!t]
\caption{Leave-P-Class-Out-CrossValidation}
\label{alg:LeavePClassOut-CrossValidation}
\LinesNumbered
\KwData{$\mathcal{D}$: classification dataset, $\mathcal{L}$: the set of labels of $\mathcal{D}$, $\alpha$: maximum number of iterations, $\mathcal{P}$: number of labels to leave out in each iteration, $\mathcal{N}$: number of cross validation folds, $\mathcal{E}$: the open-set classifier}
\KwResult{$Scores$: classification scores}
\Begin{
$\mathcal{C} \leftarrow$ {All possible combinations of $\mathcal{P}$ labels form $\mathcal{L}$}\\

\While{($\alpha > 0$) and ($\mathcal{C}\neq\emptyset$)}{
Randomly chose a combination \textit{comb} from $\mathcal{C}$\\
\textit{Leave-out-instances}$\leftarrow$ all instances of $\mathcal{D}_{l_{comb}}\mid\forall l_{comb} \in$ \textit{comb}, \textit{comb} $\subseteq\mathcal{L}$\\
\ForEach{\textit{TrainingSet}, \textit{TestSet} $\in$ $\mathcal{N}$-CrossValidation($\mathcal{D}\setminus$\textit{Leave-out-instances})}{
Train($\mathcal{E}$, \textit{TrainingSet})\\
\textit{TestSet} $\leftarrow$ \textit{TestSet} $\cup$ \textit{Leave-out-instances}\\
\textit{PredictedLabels} $\leftarrow$ Predict($\mathcal{E}$, \textit{TestSet})\\
Scores $\leftarrow $ Scores $\cup$ Statistics(\textit{PredictedLabels})\\
}
$\mathcal{C} \leftarrow \mathcal{C}\setminus$\textit{comb}\\
$\alpha \leftarrow \alpha - 1$
}
}
\end{algorithm}
\subsection{Evaluation Measures}
The natural way to evaluate classification is to use the accuracy measure which refers to the amount of correctly classified instances from the dataset. 
However, in open-set classification, the negative set can extremely outnumber the positive set which inflates the accuracy results causing an over-estimation of the classifier's performance. Moreover, the number of testing classes is (at least theoretically) undefined. F-measure (also so-called f-score), which is the harmonic mean of precision and recall, is a good alternative for open-set classification. 
We use the weighted version of f-measure as the evaluation metric for our experiments. F-measure is computed for each label, then the results are averaged, weighted by the support of each label which makes it account for label imbalance.
\subsection{Experimental Protocol and Settings}
We apply a min-max normalization on each attribute independently such that no attribute will dominate in the prediction ($x_{normalized} = \frac{x - min}{max-min}$, where min and max are the minimum and maximum values for the attribute vector). In each experiment, we use Galaxy-X to classify a dataset in a simulated open-set classification using the Leave-P-Class-Out-CrossValidation. The number of iterations $\alpha$ is set to 10 and the number of cross validations in each iteration is 5. We evaluate the classification in terms of weighted f-measure using incremental values of \textit{openness}. We compare with the gold standard multi-class classifier One-vs.-Rest \cite{rocha_2014} using a linear SVM (OvR-SVM), and with the open-set multi-class classifier One-vs.-Set SVM (OvS-SVM) \cite{walter_2013}. OvS-SVM is used with the default parameters as requested by the authors, where the generalization/specialization of the hyper-planes are performed automatically through an iterative greedy optimization of the classification risk. We also build a two-step open-set multi-class classifier, termed OCSVM+OvR-SVM, based on \cite{Landgrebe_2005} and \cite{Tax_2008} as discussed in Section \ref{sec:related_work}. In the first step of OCSVM+OvR-SVM, a one-class SVM (OCSVM) with an RBF kernel is trained on all training instances considered as a single super-class. For OCSVM, instances that deviate from the super-class are labeled as "unknown". Otherwise, the instance is passed to OvR-SVM where the latter is trained on the original training classes using a linear SVM. For Galaxy-X, we use the same closed-set classifier as OvR-SVM, OvS-SVM, and OCSVM+OvR-SVM ($i.e.$ SVM with a linear kernel). We show results of Galaxy-SVM using a fixed softening value $\delta$=-0.3 ($i.e.$ -30\% in terms of class radius). We also show results of H-Galaxy-SVM (for Hyper Galaxy-SVM) using $\delta^*$ for each openness where $\delta^*$ is obtained through a greedy search within a range of [-0.5, 0.5] with a step size of 0.1. The used distance measure for our approach is the euclidean distance. Galaxy-X is not limited to SVM but it can use other closed-set classifiers as well. In contrast, OvS-SVM is restricted to the SVM framework. Thus, we use SVM for Galaxy-X, OvR-SVM and OCSVM+OvR-SVM for consistency.

\section{Results and Discussion}\label{sec:results_discussion}
\subsection{Evaluation on Classification of Handwriting Digits} \label{subsec:evaluation_handwriting_digits}
We first evaluate our approach on a dataset of handwriting digits\footnote{http://scikit-learn.org/stable/auto\_examples/datasets/plot\_digits\_last\_image.html}. The dataset is composed of 1797 instances divided in 10 classes representing the Arabic digits. Each instance is an 8x8 image of a handwriting digit, and thus it is represented by a vector of 64 features of values between 0 and 16 respectively to the gray-scale color of the feature in the image. 
As the dataset is multidimensional, we use the t-distributed Stochastic Neighbor Embedding (t-SNE) \cite{VanDerMaaten_2014} to visualize the distribution of its instances. 
Figure \ref{fig:digits_visualization} shows a t-SNE visualization of the dataset where each data point is colored according to its ground truth class membership. We highlight the separability of the clusters where each class can approximately be completely distinguished from the rest of the dataset.
\begin{figure}[!t]
    \centering
        \includegraphics[width=.35\linewidth]{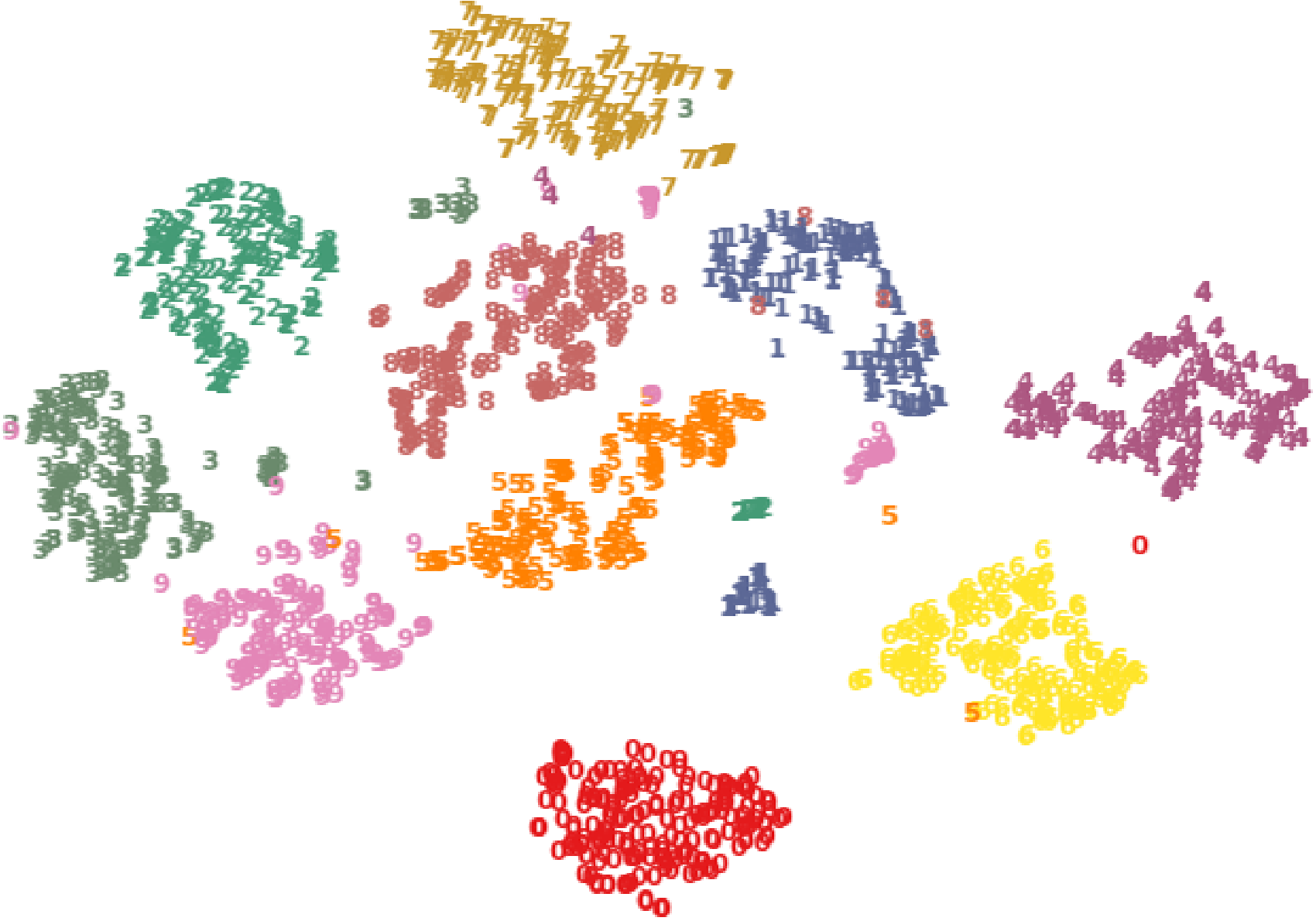}
    \caption{T-SNE visualization of the handwriting digits dataset. Colors are according to the ground truth class membership of the instances.}\label{fig:digits_visualization}
\end{figure}

Figure \ref{fig:digits_softening_SVM} shows f-measure results of Galaxy-SVM using different $\delta$ values in a simulated open-set classification of \textit{openness}=0.5, meaning that only 5 classes are seen in training and all the 10 classes are encountered in prediction. The obtained results are compared with those of SVM. The classification performance of SVM is very low as the classifier is at least incapable of correctly classifying the 5 classes that were unseen in training. SVM assigns one training label to all test instances of the 5 unknown classes leading to a misclassification. Galaxy-SVM highly outperforms SVM in terms of f-measure in the best case. However, with higher values of softening, the performance of Galaxy-SVM leans toward that of SVM. This is due to the effect of over-generalization since the bounding hyper-spheres become progressively larger with higher $\delta$ values until they completely overlap. In this setting, no rejection will be performed and only the local closed-set classifier ($i.e.,$ SVM) will be used to classify all instances. With lower $\delta$ values, the hyper-spheres become tighter adding more specialization to the class models. This allows Galaxy-SVM to better reject instances that do not resemble the overall distribution of training classes. However, the value of $\delta$ should be carefully specified since an over-specialization leads to a high distortion of the models making them incapable of covering the variance of training classes. The value of $\delta^*$ is the one that guarantees the highest f-measure representing the best trade-off between generalization and specialization for the classification scenario.
\begin{figure}[!t]
\centering
\includegraphics[width=0.5\textwidth]{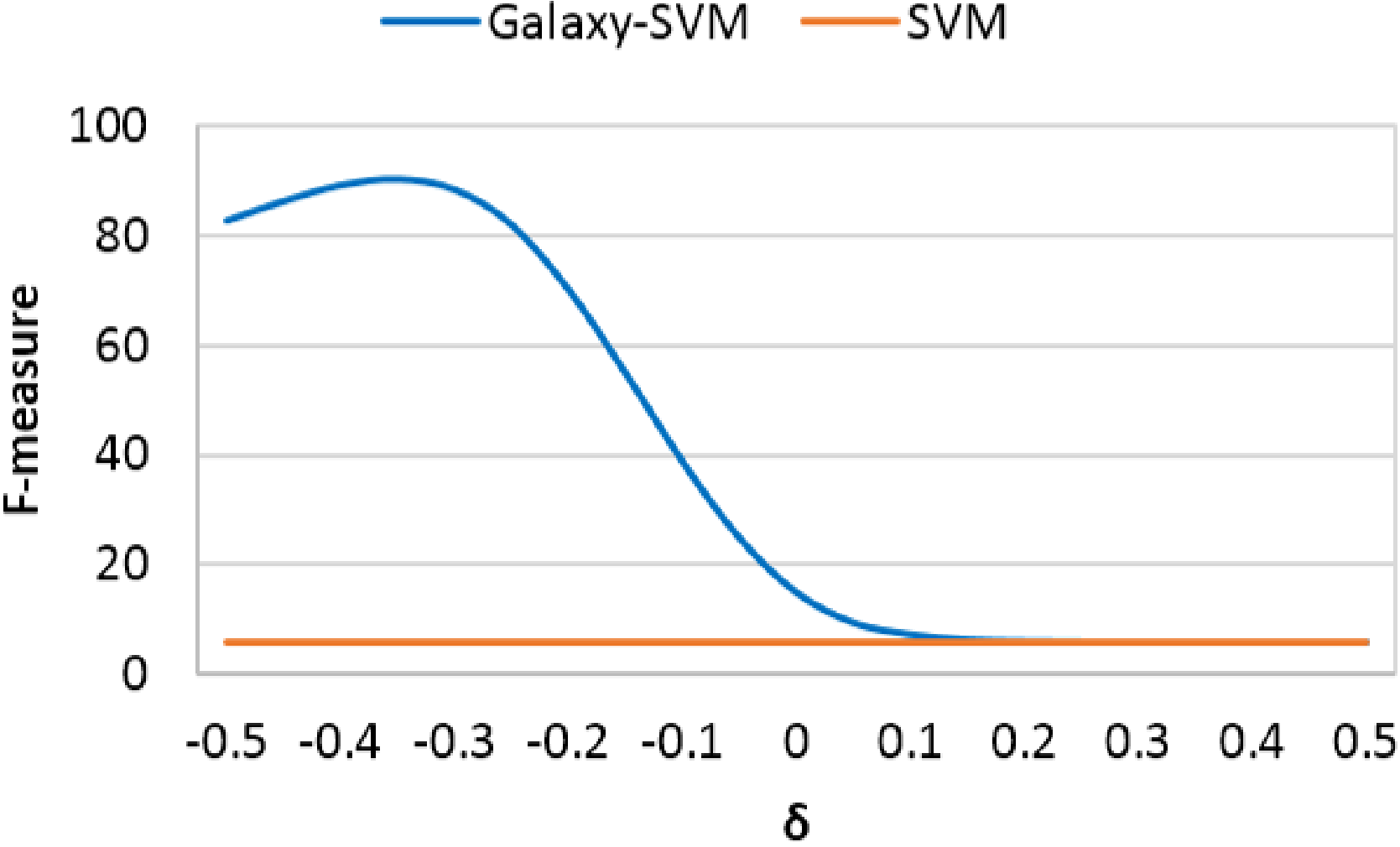}
\caption{F-measure performance in open-set classification of the handwriting digits for OvR-SVM and Galaxy-SVM with \textit{openness}=0.5 and using different $\delta$ values.}\label{fig:digits_softening_SVM}
\end{figure}
Figure \ref{fig:digits_classif} shows the classification performance in terms of f-measure for H-Galaxy-SVM (using $\delta^*$ in each iteration), Galaxy-SVM (with a fixed $\delta$ = -0.3), OvS-SVM, OCSVM+OvR-SVM, and OvR-SVM using different \textit{openness} values. The value of \textit{openness} ranges from 0 to 0.8 corresponding to a number of held-out classes ($P$) from 0 to 8 that is used in the Leave-P-Class-Out-CrossValidation. An \textit{openness} value of 0 corresponds to a closed-set classification meaning that all testing classes are seen in training. In this case, the classification performance of all classifiers are approximately the same since in the absence of rejected classes they all perform at least as good as the used closed-set classifier, $i.e.$ SVM. \textit{Openness} values from 0.1 to 0.8 correspond to open-set classification. The downward tendency of OvR-SVM is clear with higher \textit{openness} values. In fact, the more open the classification is, the more false assignments OvR-SVM will generate. In contrast, all open-set classifiers maintained higher f-measure performance than OvR-SVM due to their ability to reject instances from unseen classes. Both H-Galaxy-SVM and Galaxy-SVM ($\delta$=-0.3) outperformed all the other approaches in open-set classification scenarios. Using a fixed $\delta$ value of -0.3, Galaxy-SVM was able to give very close results to those of H-Galaxy-SVM meaning that $\delta^*\approx\delta$ in all cases. Overall, f-measure results of OCSVM+OvR-SVM were higher than those of OvS-SVM except for \textit{openness}=0.2 where they gave similar results and for \textit{openness}=0.1 where OvS-SVM outperformed OCSVM+OvR-SVM. Figure \ref{fig:digits_rejection} shows rejection f-measure results on the held-out classes for each \textit{openness} value. H-Galaxy-SVM and Galaxy-SVM scored better rejection f-measure than OvS-SVM and OCSVM+OvR-SVM in all cases providing the best trade-off between rejection-precision and rejection-recall. It is worth noting that although OCSVM+OvR-SVM and OvS-SVM used the same closed-set classifier ($i.e.$ SVM), in opposite to f-measure results for \textit{openness}=0.1 and \textit{openness}=0.2, OCSVM+OvR-SVM provided better rejection f-measure than OvS-SVM. A possible explanation for these results can be that in apposite to OCSVM+OvR-SVM, OvS-SVM performs an additional hyper-plane optimization for SVM. While OCSVM+OvR-SVM was more accurate in rejecting unknown instances than OvS-SVM, the latter provided a more accurate multi-class classification for the known classes in \textit{openness}=0.1 and \textit{openness}=0.2 settings.
\begin{figure}[!t]
\centering
    \begin{subfigure}[!t]{1\textwidth}
        \centering
        \includegraphics[width=.9\linewidth]{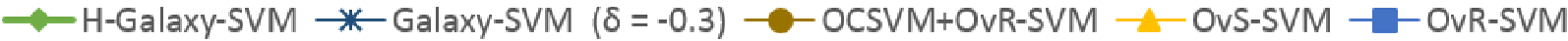}
    \end{subfigure}%

    \begin{subfigure}[!t]{0.45\textwidth}
        \centering
        \includegraphics[width=1\linewidth]{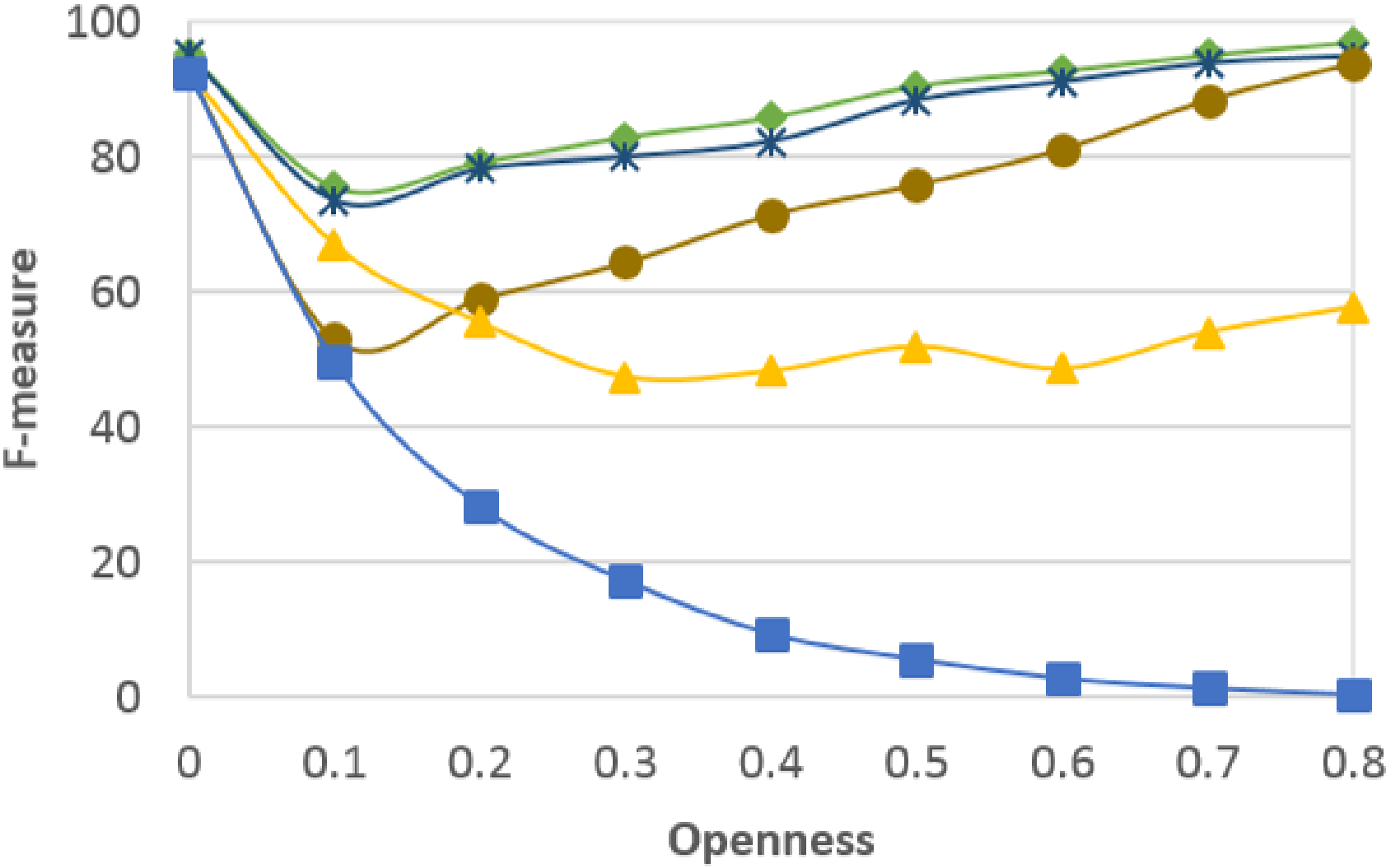}
        \caption{}\label{fig:digits_classif}
    \end{subfigure}%
    ~
    \begin{subfigure}[!t]{0.45\textwidth}
        \centering
        \includegraphics[width=1\linewidth]{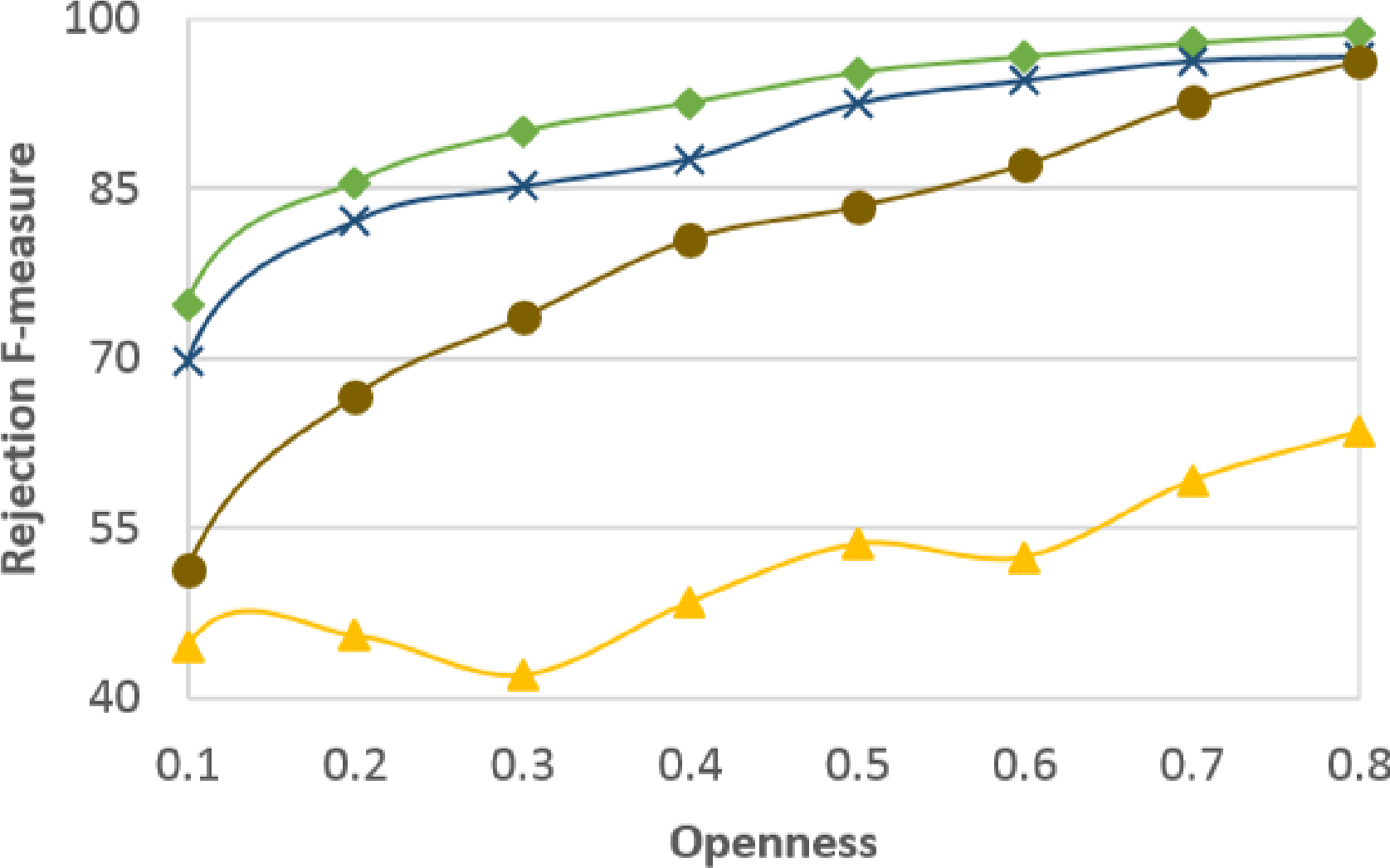}
        \caption{}\label{fig:digits_rejection}
    \end{subfigure}
\caption{F-measure (a) and rejection f-measure (b) results of H-Galaxy-SVM, Galaxy-SVM ($\delta$=-0.3), OvS-SVM, OCSVM+OvR-SVM and OvR-SVM in open-set classification of the handwriting digits dataset with different \textit{openness} values.}\label{fig:digits_classif_SVM}
\end{figure}
\subsection{Evaluation on Face Recognition}
We evaluate Galaxy-SVM on face recognition using the Olivetti faces dataset from AT\&T Laboratories Cambridge\footnote{http://www.cl.cam.ac.uk/research/dtg/attarchive/facedatabase.html}. This dataset consists of a set of 400 pictures, 10 pictures each of 40 individuals. The pictures were taken at different times, varying the lighting, facial expressions and facial details. Each picture is of a size of 64x64 resulting in a feature vector of 4096 values of gray levels. 
The task is to identify the identity of the pictured individuals. Figure \ref{fig:olivetti_visualization} shows a t-SNE visualization of the dataset. In opposite to the previous dataset, we notice a high inter-class overlapping making it difficult to isolate each class separately. With so many classes, such inter-class overlapping, and only 10 examples per class, the classification of this dataset is very challenging. Transforming this dataset into an open-set recognition task makes the classification even more challenging.
\begin{figure}[!t]
        \centering
        \includegraphics[width=.35\linewidth]{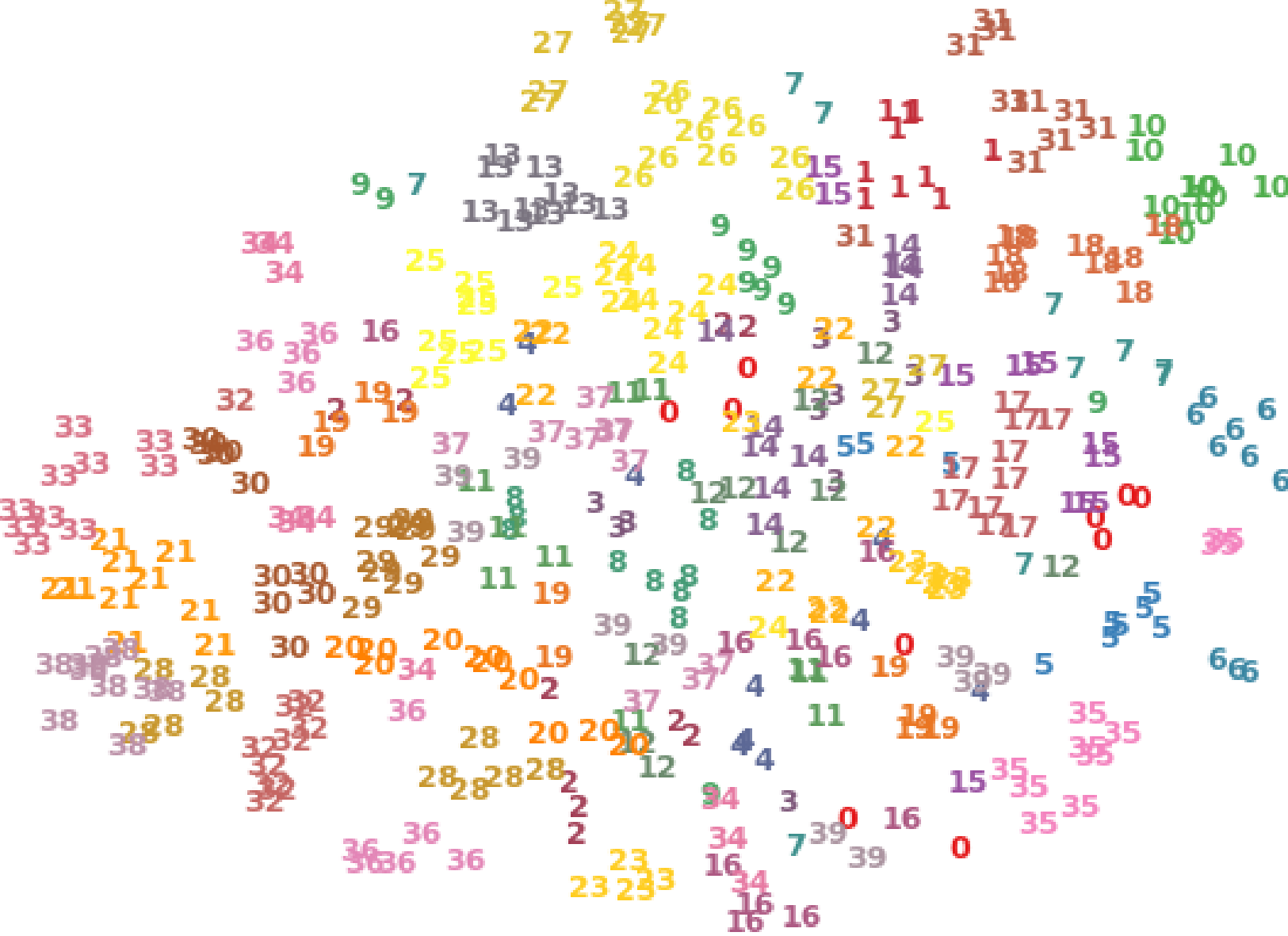}
    \caption{T-SNE visualization of the Olivetti faces dataset. Colors are according to the ground truth class membership of the instances.}\label{fig:olivetti_visualization}
\end{figure}
%

Figure \ref{fig:olivetti_classif_SVM} shows the classification f-measure (Figure \ref{fig:olivetti_classif}) and rejection f-measure (Figure \ref{fig:olivetti_rejection}) for H-Galaxy-SVM (using $\delta^*$), Galaxy-SVM (with a fixed $\delta$=-0.3), OvS-SVM, OCSVM+OvR-SVM, and OvR-SVM using different \textit{openness} values. The value of \textit{openness} ranges from 0 to 0.8 corresponding to a number of held-out classes ($P$) from 0 to 32 with a step size of 4. As shown in the figure, Galaxy-SVM handles higher values of \textit{openness} better than all the other approaches. Indeed, even at an extreme \textit{openness} value of 0.8 corresponding to only 8 training classes and 40 testing classes comprising 32 classes that were unseen in training, Galaxy-SVM was able to classify known as well as unknown class instances with high f-measure of almost 95\%. H-Galaxy-SVM outperformed all the other approaches in open-set classification cases. H-Galaxy-SVM and Galaxy-SVM ($\delta$=-0.3) gave close results for open-set classification cases except for \textit{openness}=0.1 where H-Galaxy-SVM performed better. This can be explained by the fact that in that case more generalization was needed whereas Galaxy-SVM ($\delta$=-0.3) performed a specialization of -0.3. This conclusion is supported by the f-measure result of the closed-set classifier OvR-SVM in that case, where it outperformed Galaxy-SVM ($\delta$=-0.3) with no rejection at all. Even though H-Galaxy-SVM and OvS-SVM used the same closed-set classifier (SVM) and gave very similar results in terms of rejection f-measure, H-Galaxy-SVM outperformed OvS-SVM in terms of classification f-measure in all open-set classification cases. This is due to the difference between the class representation models used in each approach. Our approach isolates each class from the rest of the classification universe from all sides. However, OvS-SVM defines two hyper-planes for each class that delimit the latter from only two sides in feature space. In this setting, the class "acceptance space" is left unlimited within the region between the hyper-planes as discussed in Section \ref{sec:related_work}. The classification performance of OCSVM+OvR-SVM compared to that of OvS-SVM is in contrast to that obtained with the handwriting digits dataset. Indeed OvS-SVM outperformed OCSVM+OvR-SVM in most open-set classification cases of the Olivetti faces dataset. This is due to the high inter-class overlapping that prevents OCSVM from efficiently isolating the training classes (as one super-class) from the overlapping unknown classes. This is clearly illustrated in the rejection f-measure results in Figure \ref{fig:olivetti_rejection} where OCSVM+OvR-SVM scored lower than all the other approaches.
\begin{figure}[!]
\centering
    \begin{subfigure}[!t]{1\textwidth}
        \centering
        \includegraphics[width=.9\linewidth]{legend_all.eps}
    \end{subfigure}%
    
    \begin{subfigure}[!t]{0.45\textwidth}
        \centering
        \includegraphics[width=1\linewidth]{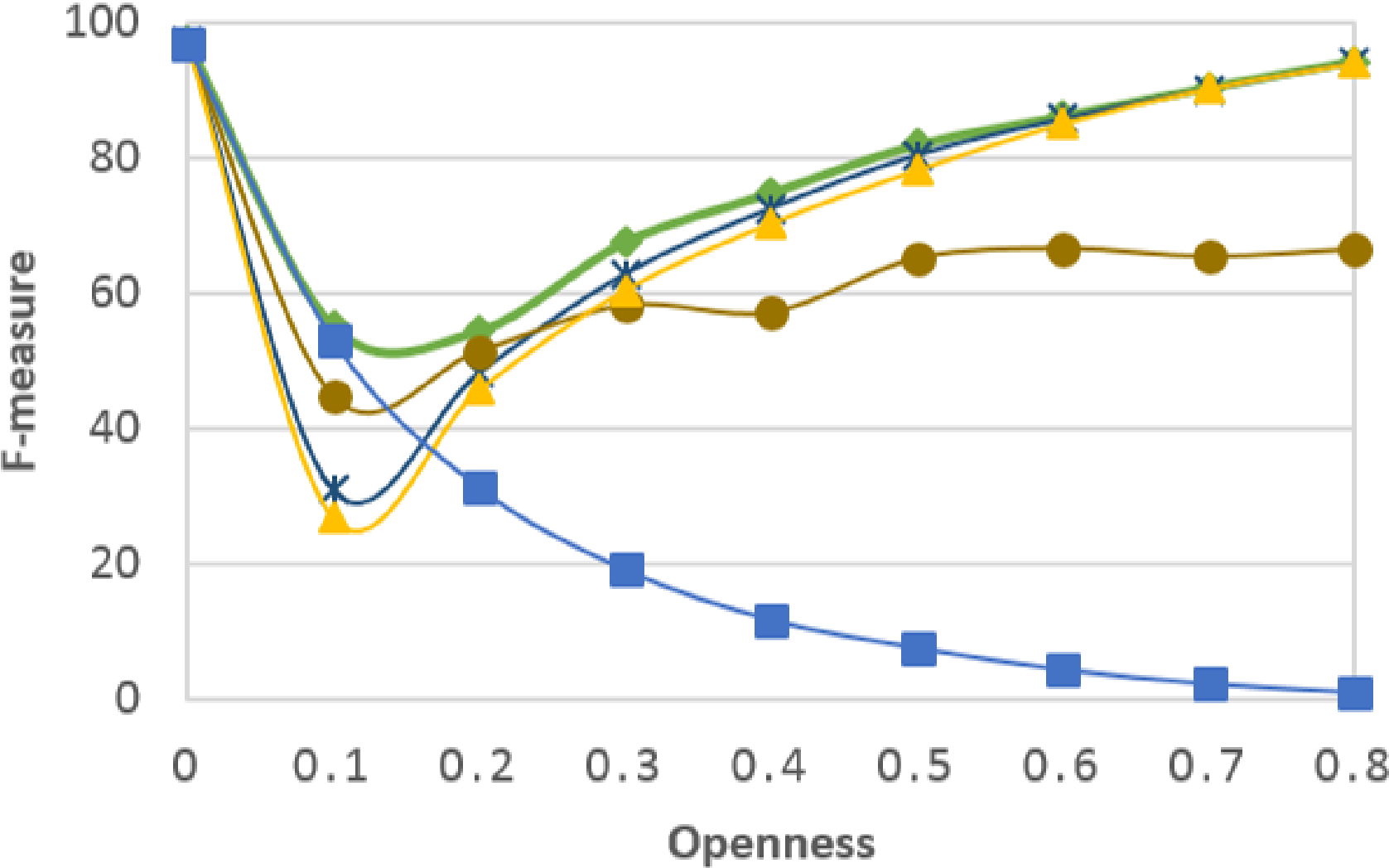}
        \caption{}\label{fig:olivetti_classif}
    \end{subfigure}%
    ~
    \begin{subfigure}[!t]{0.45\textwidth}
        \centering
        \includegraphics[width=1\linewidth]{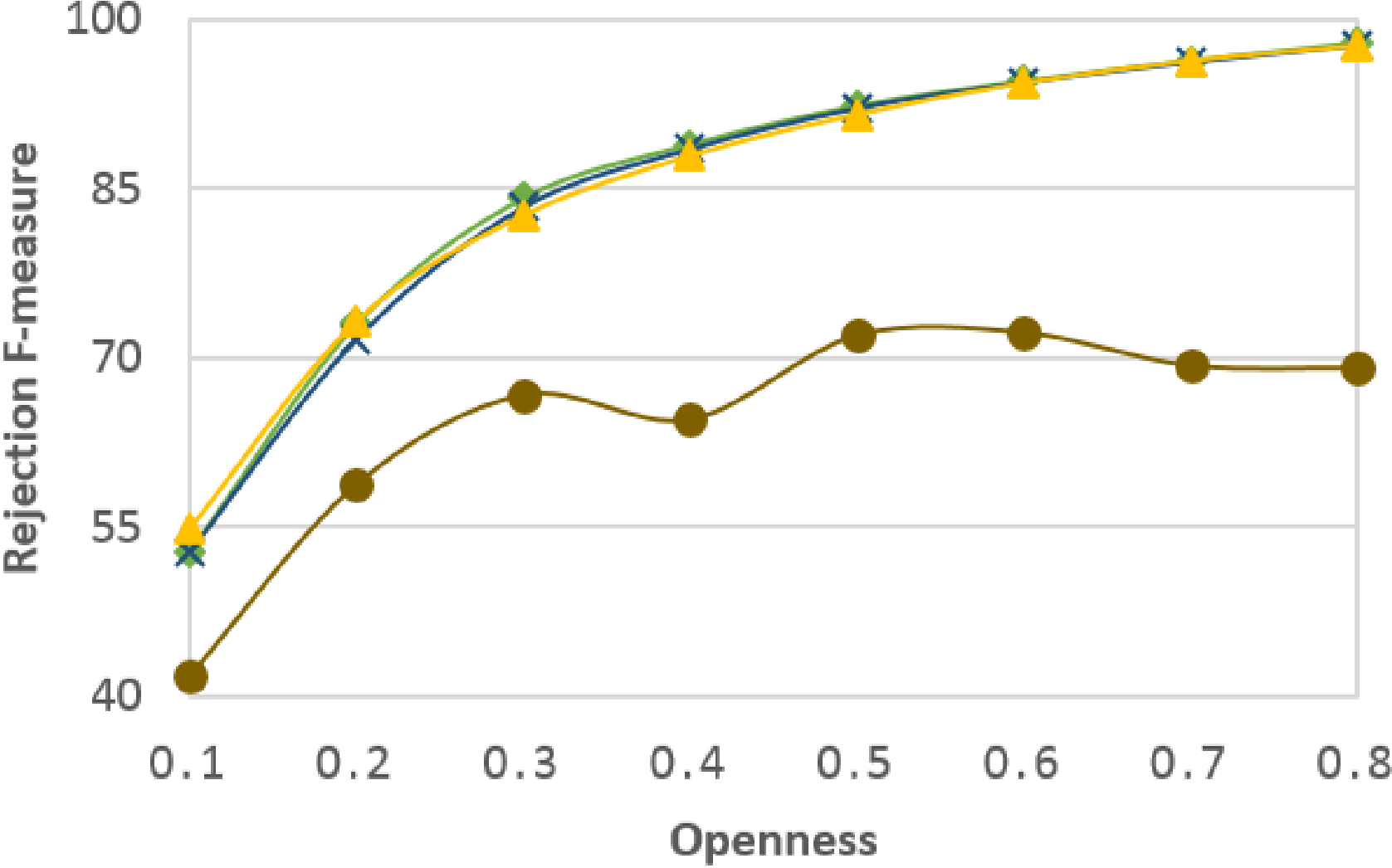}
        \caption{}\label{fig:olivetti_rejection}
    \end{subfigure}

\caption{F-measure (a) and rejection f-measure (b) results of H-Galaxy-SVM, Galaxy-SVM ($\delta$=-0.3), OvS-SVM, OCSVM+OvR-SVM and OvR-SVM in open-set classification of the Olivetti faces dataset with different \textit{openness} values.}\label{fig:olivetti_classif_SVM}
\end{figure}
\section{Conclusion}
In this paper, we addressed a fundamental problem in machine learning and data mining namely open-set classification. 
In many real-world applications where the closed-world hypothesis does not hold, it is important to prevent misclassifying instances that do not resemble any known class distribution, and raise the attention of experts to address them separately. We introduced Galaxy-X, an open-set multi-class classifier. Galaxy-X encapsulates each class with a minimum bounding hyper-sphere that allows it to distinguish instances that resemble training classes from those that are of unknown ones. Galaxy-X presents a high flexibility through a softening parameter that allows extending or shrinking class boundaries to add more generalization or specialization to the classification models. Experimental results on the classification of handwriting digits and face recognition show the efficiency of Galaxy-X in open-set classification compared to gold standard approaches from the literature. 
An interesting future work is to propose a model for non spherical classes in order to avoid the risk of over-generalization in empty regions of the hyper-sphere. 


\end{document}